\documentclass[letterpaper, 10 pt, conference]{ieeeconf}

\IEEEoverridecommandlockouts

\usepackage{tgstyle}

\usepackage{graphicx}
\usepackage{epsfig}
\usepackage{amsmath,amssymb}
\usepackage{times,mathptmx}
\usepackage{algorithm}
\usepackage{algorithmic}

\algsetup{indent=1em, linenosize=\scriptsize, linenodelimiter=.}

\usepackage{multirow}
\usepackage{booktabs}
\let\labelindent\relax
\usepackage[inline]{enumitem}

\usepackage[caption=false,font=footnotesize]{subfig}

\usepackage{xcolor}
\usepackage[dvipsnames]{xcolor}
\usepackage{colortbl}
\usepackage{url}
\usepackage[font=small]{caption}
\definecolor{myBlue}{HTML}{87CEEB}
\definecolor{myRed}{HTML}{FA8072}
\definecolor{barblue}{HTML}{A9D0E5}
\definecolor{bartan}{HTML}{EBC4A3}
\definecolor{barred}{HTML}{E6A39E}
\usepackage{tikz}
\usetikzlibrary{arrows.meta, positioning}

\usepackage[colorlinks=true,
    linkcolor=blue,
    urlcolor=blue,
    citecolor=blue]{hyperref}

\usepackage{pifont}
\usepackage[table]{xcolor}   % 注意要带 [table] option
\definecolor{babyblue}{HTML}{E3F2FD}  
\definecolor{colNGSIM}{HTML}{2196F3}      % blue
\definecolor{colNuScenes}{HTML}{4CAF50}   % green
\definecolor{colApollo}{HTML}{FF9800}     % orange

\definecolor{colNGSIM}{HTML}{E8D5F5}      % creamy lavender
\definecolor{colNuScenes}{HTML}{FFE8C8}   % creamy peach
\definecolor{colApollo}{HTML}{D5EDD5}     % creamy mint green
\newcommand{\srcNGSIM}{{\colorbox{colNGSIM}{\scriptsize\makebox[0.9em][c]{\textbf{N}}}}}
\newcommand{\srcNuScenes}{{\colorbox{colNuScenes}{\scriptsize\makebox[0.9em][c]{\textbf{Nu}}}}}
\newcommand{\srcApollo}{{\colorbox{colApollo}{\scriptsize\makebox[0.9em][c]{\textbf{A}}}}}
\usepackage{cite}

\newtheorem{theorem}{Theorem}

\newtheorem{remark}{Remark}

\newcommand{\N}{\mathbb{N}}

\begin{document}

\title{\LARGE \bf
  \raisebox{-0.2\height}{\includegraphics[height=1.5em]{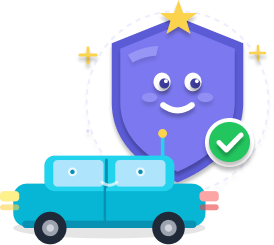}}%
  \hspace{0.5em}
Latent Dynamics--Aware OOD Monitoring for Trajectory Prediction with Provable Guarantees
}

\author{Tongfei Guo$^{1}$ and Lili Su$^{1}$%
\thanks{*This work was supported in part by the NSF CAREER award under grant~2340482. $^{1}$Department of Electrical and Computer Engineering, Northeastern University, Boston, MA, USA.
        {\tt\small \{guo.t, l.su\}@northeastern.edu}
        }%
\thanks{\scriptsize\textcopyright~2026 IEEE. Personal use of this material is permitted. Permission from IEEE must be obtained for all other uses, in any current or future media, including reprinting/republishing this material for advertising or promotional purposes, creating new collective works, for resale or redistribution to servers or lists, or reuse of any copyrighted component of this work in other works.}%
}

\maketitle

\begin{abstract}
In safety-critical Cyber-Physical Systems (CPS), trajectory prediction guides downstream planning and control. Deep learning models forecast well on validation data, but their reliability drops in out-of-distribution (OOD) scenarios driven by environmental uncertainty or rare traffic behaviors~\cite{hwang2024safeshift,filos2020can}. Such failures are often \emph{silent}: forecasts stay spatially plausible while accuracy collapses, and reported uncertainty does not rise~\cite{yang2024generalized}. Detection is hard because traffic conditions and interaction patterns keep evolving, yet the safety-critical nature of autonomous driving (AD) demands formal guarantees on detection delay and false-alarm rate. Following~\cite{guo2024building}, we reframe OOD monitoring as quickest changepoint detection (QCD), a principled statistical framework with well-established theory. We find that the evolution of prediction errors on in-distribution (ID) data is well modeled by a Hidden Markov Model (HMM). Building on this, we extend a recent cumulative Maximum Mean Discrepancy approach to our setting. The method needs no detailed prior knowledge of the post-change distribution, yet admits provable delay and false-alarm guarantees. On three real-world driving datasets, it reduces detection delay while staying robust to heavy-tailed distributions and unknown post-change conditions.
\end{abstract}

\section{INTRODUCTION}
\label{sec:introduction}

Autonomous driving (AD) systems run a closed perception-action loop. Trajectory prediction bridges perception and planning, feeding exogenous signals to downstream controllers~\cite{wei2022mpc}. When prediction quality drops, safety mechanisms can fail: control barrier functions may be invalidated, or model predictive control (MPC) problems rendered infeasible. Consider an autonomous vehicle (AV) at a snow-covered intersection---a case rare or absent in training. Snow shifts the distribution by altering road friction, visibility, and how nearby agents behave. A predictor trained under nominal conditions can then lose accuracy sharply.

Two lines of work assess prediction reliability: out-of-distribution (OOD) detection and uncertainty quantification (UQ). OOD detection works well for frame-wise classification in vision~\cite{feng2021improving}, but transfers poorly to trajectory prediction~\cite{guo2024building}: per-frame errors may be small yet accumulate as interactions evolve. UQ methods---Bayesian networks, ensembles, and Monte Carlo dropout~\cite{lakshminarayanan2017simple,gal2016dropout,liu2020energy,wang2025uncertainty,yang2024generalized}---estimate confidence across the forecasting horizon, but add heavy latency for multi-agent, multi-step prediction. Worse, their calibration assumes training data represent operating conditions, so under severe shift they can stay overconfident on OOD scenes---a real safety risk.

\begin{figure}
\vspace{0.5em}
    \centering
    \includegraphics[width=1\linewidth]{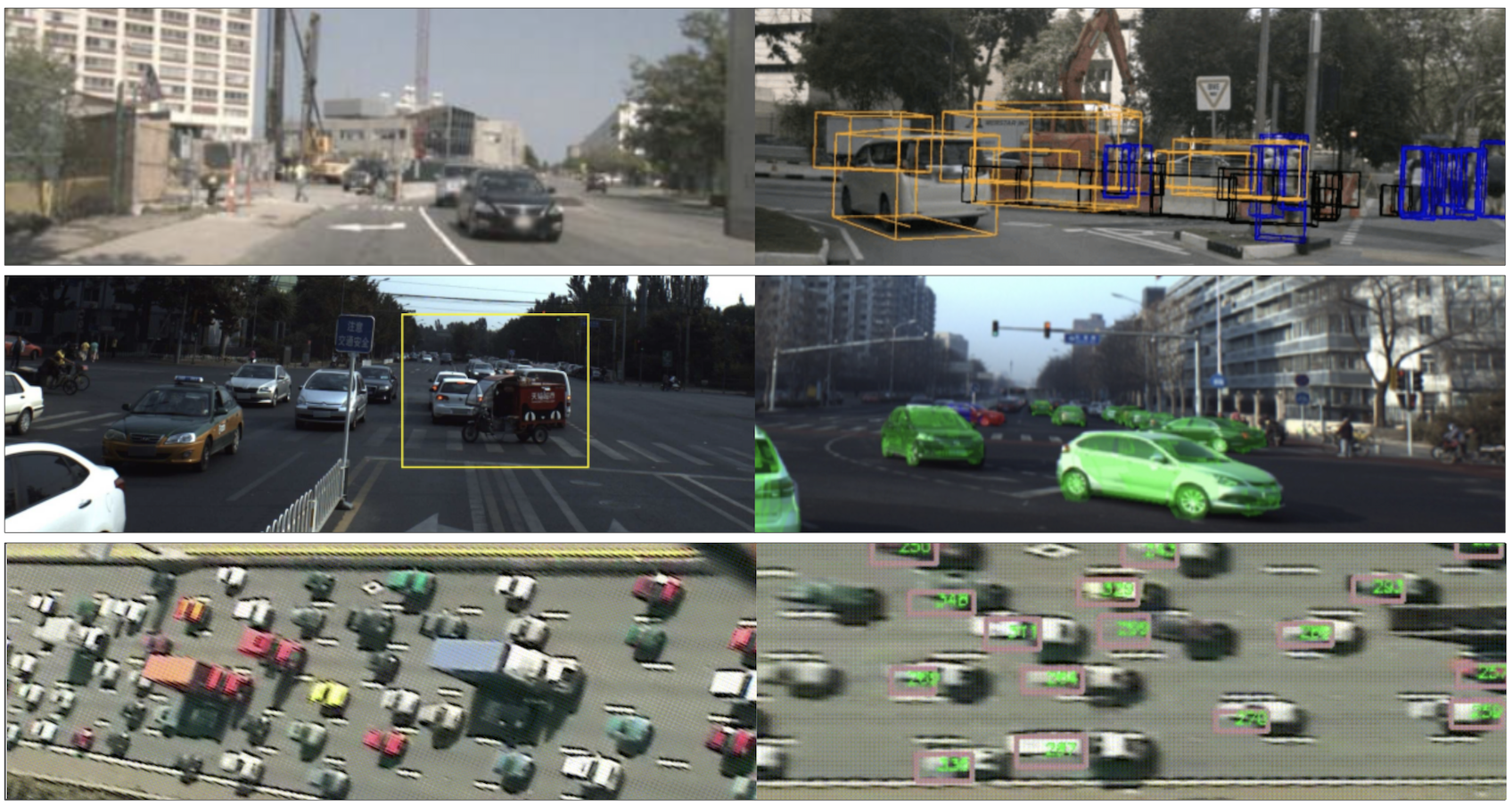}
    \caption{ { Benchmarking trajectory prediction across real-world driving datasets.} Examples from nuScenes (top), ApolloScape (middle), and NGSIM (bottom) illustrate the diversity in agent interactions and scene dynamics.
    }
    \label{fig:main}
    \vspace{-5pt}
    % \vspace{-2em}
\end{figure}

Quickest change detection (QCD) is a natural, theory-grounded fit, since it models how prediction errors evolve over time. Classical CUSUM gives minimax delay guarantees under false-alarm constraints~\cite{tartakovsky2014sequential,page1954continuous,basseville1993detection}, and recent work shows CUSUM works well for AD~\cite{guo2024building}. Prediction errors in AD are also \emph{multimodal}~\cite{guo2025dynamic}: a low-error mode for structured, predictable contexts, and a high-error mode for complex interactions. These modes switch latently as the driving context changes, creating temporal dependence that a unimodal distribution cannot capture. Hidden Markov Models (HMMs) cover a broad family of such dynamics through their transition probabilities and emission distributions. They abstract these latent modes well~\cite{mor2021systematic,fuh2018asymptotic} and explain real-world patterns across datasets (Section~\ref{sec:hmm_formulation}).

Building on this, we extend the data-driven cumulative Maximum Mean Discrepancy framework of~\cite{zhang2024data} to AD. Existing QCD methods for HMMs~\cite{fuh2015quickest} need exact pre- and post-change specifications, which are rarely available at deployment. We instead build the \emph{non-stationary (latent) Dynamic CuSum Maximum Mean Discrepancy (DC-MMD)} detector on the theory of~\cite{zhang2024data}. It uses no separate change-free reference sequence; the reference comes from second-order statistics of the ID data. Pairing consecutive errors $x_t=(e_{t-1},e_t)$ preserves the Markovian transition structure, so the detector still reacts to transition-regime shifts that a first-order reference would miss.

\noindent
We summarize our contributions in the following:

\ding{182} \textbf{Non-stationary (latent) dynamics-aware detection.}
Error regimes switch frequently between low and high (Fig.~\ref{fig:latent-dynamic}), so a single stationary model conflates benign mode switching with true OOD. We cast monitoring as QCD over an HMM-driven prediction-error process,
designed to absorb routine switching and react only to genuine shifts, with provable delay and false-alarm guarantees.

\ding{183} \textbf{Robustness under heavy tails.}
Detection stays stable from light- to heavy-tailed errors, even when tail assumptions are violated.

\ding{184} \textbf{Robustness to unknown post-change.} DC-MMD needs no prior knowledge of the post-change distribution, yet matches methods designed for that setting.

\ding{185} \textbf{Real-world validation.}
On nuScenes, NGSIM, and ApolloScape, DC-MMD consistently improves over sequential and UQ-based baselines at real-time cost.

\section{Related Work}

\subsection{Trajectory Prediction under Distributional Shift}
\label{subsec:traj_shift}

Modern trajectory predictors---from graph-based Social-LSTM~\cite{alahi2016social}, GRIP++~\cite{li2020gripplus}, and Trajectron++~\cite{salzmann2020trajectron++} to attention-based forecasters~\cite{giuliari2021transformer}---reach strong ID accuracy but stay brittle under shifts in traffic density, topology, and interaction patterns~\cite{filos2020can,hwang2024safeshift}. These shifts are often \emph{silent}: forecasts look spatially coherent yet are inaccurate, with no matching rise in reported uncertainty~\cite{yang2024generalized}. Common robustness tools---Deep Ensembles~\cite{lakshminarayanan2017simple}, MC-Dropout~\cite{gal2016dropout}, and energy scores~\cite{liu2020energy}---measure uncertainty at the \emph{single-frame} level, ignoring the \emph{sequential, non-stationary} structure of deployment-time errors. We close this gap by modeling that error process as an HMM and pairing a data-driven MMD statistic with a CUSUM rule, catching reliability loss as it unfolds rather than frame by frame.

\subsection{OOD Detection and Sequential Change Detection}
\label{subsec:ood_qcd}
Frame-level OOD detectors~\cite{hendrycks2017baseline,wiederer2023joint} accumulate no temporal evidence. Classical sequential tests like CUSUM~\cite{page1954continuous} assume stationary or known noise, which breaks down in urban driving~\cite{basseville1993detection}. Recent variants~\cite{ahad2024das} relax some assumptions but still need a pre-specified post-change family, and~\cite{guo2025dynamic} adds mode-aware thresholds yet leaves the realistic latent non-stationary case open. Deep OOD detectors---energy-based~\cite{liu2020energy}, density/likelihood-based~\cite{ren2019likelihood}, and feature-distance scores~\cite{lee2018simple}---and learned sequential anomaly detectors~\cite{chakraborty2023structural} are not directly comparable: they score single frames or need model introspection, retraining, or internal features, give no stopping rule, and offer no delay/false-alarm guarantee. DC-MMD instead runs on the scalar error stream $\{e_t\}$ alone, with provable WADD/MTFA control, allowing a fair matched-MTFA comparison against the sequential baselines we adopt.

\section{Problem Formulation}
\label{sec:problem}

\subsection{Trajectory Prediction}

Consider an AV in a dynamic traffic environment. Time is slotted into steps set by the sampling frequency. Let $S^t$ denote the scene at step $t \in \N$, comprising the trajectories of neighboring vehicles and the map over the most recent $H$ steps, i.e.,
\[
S_t = \left\{ \{ p_{t-H+1,i}, \cdots, p_{t,i}\}_{i=1}^{n_t},  ~ {\mathcal{M}_{t-H+1}, \cdots, \mathcal{M}_t} \right\},
\]
where $n_t$ is the number of neighboring vehicles at time step $t$,
$H$ is the observation window length,
$p_{r,i}$ is the trajectory position of vehicle $i$ at time step $r$, and $\mathcal{M}_r$ is the map for $r=t-H+1, \cdots, t$.

Let $f^\theta$ be a given trajectory predictor, and $\mathcal{D}$ the corresponding training dataset, which is a collection of traffic scenes.
The trajectory predictor maps an encountered traffic scene into predicted trajectories of neighboring vehicles, i.e.,
\[
f^{\theta}(S_t) = \{\hat{p}_{t+1,i}, \cdots, \hat{p}_{t+L,i}\}_{i=1}^{n_t},
\]
where $L\in \N$ is the prediction window.

\smallskip
\noindent\textbf{Prediction Error.}
We measure the performance of the given trajectory prediction $f^\theta$ through widely adopted metrics such as Average Displacement Error (ADE), Final Displacement Error (FDE), and Root Mean Squared Error (RMSE) \cite{guo2024building}.

Let $v_t$ denote a target vehicle, which can be arbitrarily chosen from the neighboring vehicles.
The ADE averages pointwise $\ell_2$ prediction errors:
$
\text{ADE}_t = \frac{1}{L}\sum_{\ell=1}^{L} \bigl\|\hat{p}_{t+\ell, v_t} - p_{t+\ell, v_t}\bigr\|_2.
\label{eq:error ADE}
$
Similarly, FDE measures the $\ell_2$ error in the final position:
$
\label{eq:error FDE}
\text{FDE}_t = \|\hat{p}_{t+L, v_t} - p_{t+L, v_t}\|_2.
$
RMSE is the square root of the averaged squared $\ell_2$ errors:
$
\label{eq:error RMSE}
\text{RMSE}_t = \sqrt{\frac{1}{L}\sum_{\ell=1}^{L} \bigl\|\hat{p}_{t+\ell, v_t} - p_{t+\ell, v_t}\bigr\|_2^2}.
$

\subsection{OOD Detection as Quickest Change-Point Detection}
\label{sec:sequentialQCD}
Prediction error can be observed online as true trajectories are revealed. Our framework monitors this error sequence rather than raw trajectories or internal states, making it lightweight and agnostic to the choice of predictor $f^{\theta}$.

Let $\{e_t\}_{t \geq 1}$ denote the sequence of prediction errors produced by the deployed predictor $f^\theta$, where $e_t$ may be instantiated by any of the above metrics (ADE, FDE, or RMSE).
An OOD event may begin to emerge at an unknown time step $\tau^* \in \mathbb{N} \cup \{\infty\}$; when $\tau^* = \infty$, no OOD event ever occurs. We refer to $\tau^*$ as the changepoint.
Formally,
\[
e_t \sim
\begin{cases}
f_t(\cdot), & t < \tau^*, \\
g_t(\cdot), & t \geq \tau^*,
\end{cases}
\label{eq:change_model}
\]
where $f_t$ denotes the ID error density, and $g_t$ denotes the OOD error density.
The non-stationarity of $f_t$ and $g_t$ arises from the inherent motion dynamics and environmental evolution that are unique to AVs (see Section~\ref{sec:hmm_formulation}).

\smallskip
\noindent\textbf{Evaluation.}
We adopt the minimax framework for QCD~\cite{lorden1971procedures} and evaluate detection performance through two competing objectives: detection delay and false alarm control.

We quantify detection delay via a form of \emph{worst-case average detection delay} (WADD) \cite{zhang2024data}:
\begin{equation}
\mathrm{WADD}(\tau)
:= \sup_{\tau^* \geq 1}
\mathbb{E}_{\tau^*} \!\left[ (\tau - \tau^*)^{+} \mid \mathcal{F}_{\tau^*-1} \right],
\label{eq:wadd}
\end{equation}
where $\mathcal{F}_{\tau^*-1}$ is the natural filtration that incorporates all randomness up to the end of time step $\tau^*-1$,
$\mathbb{E}_{\tau^*}[\cdot]$ denotes expectation under the measure $\mathbb{P}_{\tau^*}$,
and the supremum ranges over all possible realizations of the changepoint.
This metric helps us guard against the adversarial choice of the changepoint, which is essential for safety-critical applications.

To characterize detector stability under nominal conditions, we use the \emph{mean time to false alarm} (MTFA):
\begin{equation}
\mathrm{MTFA}(\tau) := \mathbb{E}_{\infty}[\tau],
\label{eq:mtfa}
\end{equation}
where $\mathbb{E}_{\infty}[\cdot]$ denotes expectation when no OOD event ever occurs.
Following Lorden's optimality~\cite{lorden1971procedures}, we seek a detection method that achieves short delay subject to a prescribed false alarm constraint:
\begin{equation}
\tau^{*} \in \operatorname*{arg\,min}_{\tau}\;\mathrm{WADD}(\tau)
\quad \text{subject to} \quad
\mathrm{MTFA}(\tau) \geq \gamma,
\label{eq:lorden_problem}
\end{equation}
where $\gamma > 0$ is the user-specified tolerable false alarm rate.
Setting a larger $\gamma$ often enforces stricter false alarm control at the cost of increased detection delay, whereas choosing a smaller $\gamma$ prioritizes rapid detection.

\section{Modeling the Non-Stationarity of Prediction Errors}
\label{sec:hmm_formulation}
\begin{table}[t!]
\vspace{0.5em}
\centering
\caption{
Latent error dynamics across driving datasets, modeled as HMMs with emission parameters and transition probability.
Emission parameters: $\Delta\mu$ (mean gap between modes) and $\bar{\sigma}^{2}$ (within-mode variance) characterize detection difficulty as {Easy} (large $\Delta\mu$, small $\bar{\sigma}^{2}$), {Moderate (Mod.)}, or {Hard} (small $\Delta\mu$, large $\bar{\sigma}^{2}$).
Transition parameter: $p$ (mode-switching probability) reflects how often the latent state changes.
Each marker denotes the dataset source:
% $\bullet$\,=\,NGSIM, $\blacktriangle$\,=\,nuScenes, $\star$\,=\,ApolloScape.
\srcNGSIM\,=\,NGSIM,\quad \srcNuScenes\,=\,nuScenes,\quad \srcApollo\,=\,ApolloScape.
}
\label{tab:signature}
\footnotesize
\setlength{\tabcolsep}{4pt}
\begin{tabular}{llccc}
\toprule
\multirow{2}{*}{\textbf{Scene}} & \multirow{2}{*}{\textbf{Src}}
  & \multicolumn{2}{c}{\textbf{Emission}} & \textbf{Transition} \\[-0.5ex]
\cmidrule(lr){3-4} \cmidrule(lr){5-5}
& & $\Delta\mu$ & $\bar{\sigma}^{2}$ & $p$ \\
\midrule
\rowcolor{babyblue}
\multicolumn{5}{l}{\textit{Highway --- structured, low-density}}
\\
Car-following (platoon)   & \srcNGSIM 
  & Mod. & Easy & 0.15 \\
Congested stop-and-go     & \srcNGSIM 
  & Mod. & Mod. & 0.32 \\
\midrule
\rowcolor{babyblue}
\multicolumn{5}{l}{\textit{Mixed Urban --- transitional, Mod. complexity}} \\
On-ramp merge             & \srcNuScenes\;\srcApollo 
  & Mod. & Mod. & 0.55 \\
Signalised intersection   & \srcNuScenes\;\srcApollo 
  & Mod. & Mod. & 0.50 \\
\midrule
\rowcolor{babyblue}
\multicolumn{5}{l}{\textit{Unstructured Urban --- unstructured, high-density}} \\
Unsignalised roundabout   & \srcNuScenes\;\srcApollo
  & Hard & Mod. & 0.87 \\
Unsignalised intersection & \srcApollo
  & Hard & Hard & 0.88 \\
Pedestrian-heavy zone     & \srcApollo
  & Hard & Hard & 0.91 \\
\bottomrule
\end{tabular}
\vspace{-5pt}
% \vspace{-2em}
\end{table}

Prediction errors are non-stationary even without OOD events. As~\cite{guo2025dynamic} observes, the error process alternates between two latent modes: a \emph{low-risk mode} that stays stable under routine flow, and a \emph{high-risk mode} that emerges during dense traffic, occlusions, or abrupt maneuvers. The high-risk mode is not an isolated spike; it persists over several steps as the scene evolves. HMMs capture such dynamics naturally: transition probabilities set how long each regime lasts, and emission distributions set the error magnitude within it. Empirically, these transitions show clear Markovian structure and fit HMMs well (Section~\ref{subsec:empirical_evidence}).

\begin{figure*}[t!]
\vspace{0.6em}
    \centering
    \includegraphics[width=1\linewidth]{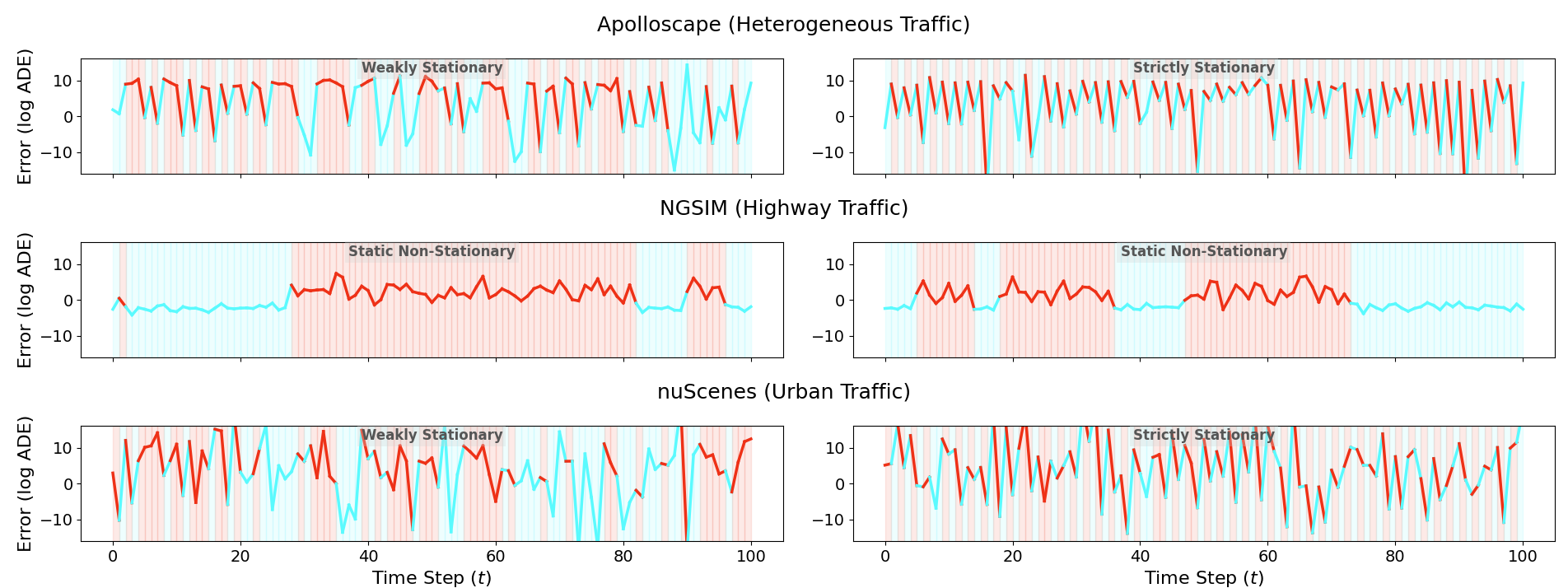}
    \caption{Non-stationary switching behavior of prediction errors across three driving environments: NGSIM (Highway), nuScenes (Mixed Urban), and ApolloScape (Unstructured Urban). The $x$-axis represents the time step ($t$) and the $y$-axis denotes the error (log ADE). Preliminary  results on other metrics (i.e., FDE and RMSE) show similar behaviors.
    Background shading indicates the inferred latent error mode: Low-Error Mode ({\color{myBlue} Blue}) and High-Error Mode ({\color{myRed}Red}). The two modes are separated via MAP estimation under a Gaussian mixture model posterior, which assigns each observation to the component with the highest posterior responsibility.}
    \label{fig:latent-dynamic}
    \vspace{-1em}
\end{figure*}

\subsection{Empirical Evidence for Dynamic-Switching Errors}
\label{subsec:empirical_evidence}

We analyze error dynamics on diverse real-world datasets:
NGSIM~\cite{fhwa2020} (highway),
nuScenes~\cite{caesar2020nuscenes} (mixed urban),
and ApolloScape~\cite{huang2018apolloscape}
(unstructured urban).
We fit Gaussian-emission HMMs with different
numbers of latent states.
A two-state Gaussian-emission HMM gives the best fit throughout, matching the bimodal finding of~\cite{guo2025dynamic}. This holds across all three datasets and both architectures---graph-based GRIP++~\cite{li2020gripplus} and attention-based FQA~\cite{kamra2020multi}---so the two modes reflect physically distinct regimes (routine low-error flow vs.\ complex high-error interactions), not a quirk of one dataset or predictor. The choice is deliberate: two states are the minimal structure that separates these regimes while keeping the reference $\lambda^{(0)}$ estimable from limited ID data.

Table~\ref{tab:signature} summarizes the fitted parameters in terms of
emission parameters ($\Delta\mu$, $\bar{\sigma}^2$) and transition probability ($p$). Our key findings are as follows:

\vspace{0.2em}
\noindent
\textbf{Highway scenes} (NGSIM) exhibit clearly separated low- and
high-risk modes with stable dynamics and relatively rare switching
($p\!\approx\!0.15$--$0.32$ across scene types, see Table~\ref{tab:signature}). Most time steps remain in the low-risk mode,
while high-risk episodes are typically triggered by abrupt maneuvers
such as cut-ins or braking waves.

\noindent
\textbf{Urban scenes} (nuScenes and ApolloScape) show more frequent
switching ($p= 0.5$--$0.9$), reflecting rapidly evolving agent-agent interactions and agent-lane interactions.

Fig.~\ref{fig:latent-dynamic} shows that latent dynamics differ even across scenes in one dataset; still, every dataset and predictor switches non-stationarily, only at different rates. These shared traits---latent modes with persistent but variable transitions---motivate a probabilistic abstraction, which we formalize next as an HMM over the error sequence.

\subsection{HMM Error Model}
\label{subsec:hmm_model}

To capture temporal dependence in prediction errors in the absence of OOD events, we model the error sequence
$\{e_t\}_{t\ge1}$ using a HMM.
The key idea is that prediction errors arise from multiple latent operating
conditions, which evolve over time.

Let $Z_t \in \{L,H\}$ denote the latent error mode at time $t$, where
state $L$ corresponds to a nominal low-error mode and state $H$ represents
an elevated-error mode. We model the random evolution $\{Z_t: t=1, ...\}$ as a Markov chain with
\[
P_{ij} = \Pr(Z_{t+1}=j \mid Z_t=i), ~~~ \forall  ~ i, j\in \{L, H\}.
\]
We assume $P_{ij}>0$ for $i, j\in \{L, H\}$.
Conditioned on the latent state, the observed error at each time step is drawn
independently from a state-specific emission distribution:
\[
e_t \mid Z_t = k \;\sim\; h_k, \quad k \in \{L, H\},
\]
where $h_L$ and $h_H$ denote the low-error and high-error emission densities,
respectively.

\subsection{Pre- and Post-Change Error Models}
\label{subsec:regime_change}
For generality, we allow the post-change distribution $g_t$ to follow HMM as well.
Let $\Theta^{(0)}=(P^{(0)},h_L^{(0)},h_H^{(0)})$ and  $\Theta^{(1)}=(P^{(1)},h_L^{(1)},h_H^{(1)})$
denote the pre-and post-change HMM parameters, respectively.
This formulation captures two types of distribution shifts:

\textit{Emission shift:}
$h_k^{(0)} \neq h_k^{(1)}$ for some $k \in \{L,H\}$.

\textit{Transition shift:}
$P^{(0)} \neq P^{(1)}$, reflecting altered switching
behavior between error modes.

In our experiments (Sections~\ref{sec:experiments} and~\ref{sec:results}), $\Theta^{(0)}$ is learned from the predictor's training data, with $P_{ij}^{(0)}>0$ for every entry. We assume no prior knowledge of $\Theta^{(1)}$, only $P_{ij}^{(1)}>0$ for $i,j \in \{L, H\}$. Rather than constrain each shift type separately, we require the joint distributions $\{f_t\}$ and $\{g_t\}$ to be distinguishable, i.e.,
\[
D_{KL} (f_1, f_2, ...|| g_1, g_2, ...) >0, ~ D_{KL} (g_1, g_2, ...||f_1, f_2, ...) >0,
\]
where $D_{KL}(\cdot||\cdot)$ is the Kullback--Leibler (KL) divergence.

\section{Method: Kernel-based Detection}
\label{sec:method}
To capture the information in the mode dynamics across time steps, following \cite{zhang2024data}, we monitor second-order prediction errors $x_t := (e_{t-1}, e_t) \in \mathbb{R}^2.$
Since each entry of $P^{(0)}$ is positive, the limiting distribution of $x_t$ under $\Theta^{(0)}$ exists and is unique.
We denote the distribution as
$\lambda^{(0)}$, estimated from training data.
Similarly, the limiting distribution of $x_t$ under $\Theta^{(1)}$, denoted by $\lambda^{(1)}$, exists and is unique.

\vspace{-0.1em}
In real AD systems, uncertainty limits what we know about the post-change distribution, while delay and false-alarm guarantees are still required. We adapt~\cite{zhang2024data} to meet both needs, as detailed in Algorithm~\ref{alg:hmm_kernel_cusum}. The original method uses a companion change-free reference sequence; since none exists in our setting, we build the reference from second-order statistics of ID data.

DC-MMD works as follows. The error sequence is split into non-overlapping blocks of length $m$ (a tuning parameter). For each block, we measure the deviation between its empirical second-order distribution and the reference $\lambda^{(0)}$, and declare an OOD event once the cumulative deviation passes a threshold $b$. The offset $\zeta$ tunes sensitivity to observation noise.

\begin{algorithm}[ht]
\caption{Non-stationary (Latent) Dynamic CuSum Maximum Mean Discrepancy (DC-MMD)}
\label{alg:hmm_kernel_cusum}
\begin{algorithmic}[1]
\STATE \textbf{Input:} second-order error distribution of ID data $\lambda^{(0)}$, block length $m$, offset $\zeta$, threshold $b$
\STATE \textbf{Initialize:} $W_0 \gets 0$, $\mathcal{B}_0\gets \mathcal{O}$
\FOR{$t=1, \cdots$}
\STATE observe a prediction error $e_t$
\STATE construct the second-order sample $x_t=(e_{t-1}, e_t)$
\STATE $\mathcal{B}_t \gets \mathcal{B}_{t-1}\cup \{x_t\}$
\IF{$|\mathcal{B}_t| \mod m = 0$}
    \STATE construct the empirical distribution $\widehat{\mu}_{\lfloor t/m \rfloor}$
    \STATE compute $D(\widehat{\mu}_{\lfloor t/m \rfloor}, \lambda^{(0)})$
    \STATE $W_{\lfloor t/m \rfloor} \gets \max\{0,\,W_{\lfloor t/m \rfloor -1} + D(\widehat{\mu}_{\lfloor t/m \rfloor}, \lambda^{(0)}) -\zeta\}$
    \IF{$W_{\lfloor t/m \rfloor}>b$}
        \STATE declare an OOD event and {\bf break}
    \ENDIF
    \STATE $\mathcal{B}_{t} \gets \mathcal{O}$
\ENDIF
\ENDFOR
\end{algorithmic}
\end{algorithm}

% \begin{algorithm}[t]
% \small
% \setlength{\abovedisplayskip}{2pt}
% \caption{Dynamic CuSum MMD (DC-MMD)}
% \label{alg:hmm_kernel_cusum}
% \begin{algorithmic}[1]
% \setlength{\itemsep}{-1pt}
% \STATE \textbf{Input:} ID second-order error distribution $\lambda^{(0)}$, block length $m$, offset $\zeta$, threshold $b$
% \STATE \textbf{Initialize:} $W_0 \gets 0$, $\mathcal{B}_0\gets \emptyset$
% \FOR{$t=1, \cdots$}
% \STATE observe prediction error $e_t$; set $x_t=(e_{t-1}, e_t)$
% \STATE $\mathcal{B}_t \gets \mathcal{B}_{t-1}\cup \{x_t\}$
% \IF{$|\mathcal{B}_t| \bmod m = 0$}
%     \STATE construct $\widehat{\mu}_\bk$ and compute $D(\widehat{\mu}_\bk, \lambda^{(0)})$
%     \STATE $W_\bk \gets \max\{0,\,W_{\bk-1} + D(\widehat{\mu}_\bk, \lambda^{(0)}) -\zeta\}$
%     \IF{$W_\bk>b$}
%         \STATE declare an OOD event and {\bf break}
%     \ENDIF
%     \STATE $\mathcal{B}_{t} \gets \emptyset$
% \ENDIF
% \ENDFOR
% \end{algorithmic}
% \end{algorithm}

Formally, let $\mathcal{B}_t = \{e_{(t-1)m +1}, \cdots, e_{tm}\}$ denote the $t$-th block of the partitioned errors, which gives $m-1$ consecutive second-order samples, i.e., $x_t^1, \cdots, x_t^{m-1}$. Let $\hat{\mu}_t$ denote the empirical second-order distribution over the $t$-th block.

To compare $\hat{\mu}_t$ and $\lambda^{(0)}$, we use MMD in a Reproducing Kernel Hilbert Space (RKHS).
Let $k(\cdot,\cdot)$ be a bounded kernel with associated RKHS $\mathcal{H}$.
The MMD between $\hat{\mu}_t$ and $\lambda^{(0)}$, denoted by $D(\hat{\mu}_t, \lambda^{(0)})$, can be computed as
\begin{align*}
D^2(\hat{\mu}_t, \lambda^{(0)})
&= \mathbb{E}_{X, \bar{X}\sim \hat{\mu}_t} [k(X, \bar{X})] + \mathbb{E}_{Y, \bar{Y}\sim \lambda^{(0)}} [k(Y, \bar{Y})] \\
& \quad - 2  \mathbb{E}_{X\sim \hat{\mu}_t, Y\sim \lambda^{(0)}}[k(X,Y)],
\end{align*}
where $X$, $\bar{X}$ are independent copies that follow $\hat{\mu}_t$, and $Y$, $\bar{Y}$ are independent copies that follow $\lambda^{(0)}$.

\begin{table*}[t!]
\vspace{0.6em}
\centering
\caption{{OOD detection performance across datasets.} Higher AUROC ($\uparrow$) and lower FPR@95 ($\downarrow$) indicate better performance.}
\label{tab:rq1-performance}
\begin{tabular}{ll ccc ccc}
\toprule
\textbf{Category} & \textbf{Method} & \multicolumn{3}{c}{\textbf{AUROC} $\uparrow$} & \multicolumn{3}{c}{\textbf{FPR@95} $\downarrow$} \\
\cmidrule(lr){3-5} \cmidrule(lr){6-8}
& & \textbf{NGSIM} & \textbf{nuScenes} & \textbf{ApolloScape} & \textbf{NGSIM} & \textbf{nuScenes} & \textbf{ApolloScape} \\
& & (Highway) & (Mixed urban) & (Unstructured Urban) & (Highway) & (Mixed urban) & (Unstructured Urban) \\ \midrule
Point-wise & NLL \cite{lee2018simple} & 0.71 & 0.64 & 0.58 & 0.62 & 0.71 & 0.78 \\
(UQ-based) & lGMM \cite{wiederer2023joint} & 0.74 & 0.69 & 0.62 & 0.57 & 0.65 & 0.73 \\ \midrule
Sequential & G-CUSUM \cite{page1954continuous} & 0.82 & 0.76 & 0.69 & 0.44 & 0.52 & 0.61 \\
& GMM-CUSUM \cite{guo2024building} & 0.85 & 0.79 & 0.73 & 0.41 & 0.48 & 0.56 \\
& Mode-CUSUM \cite{guo2025dynamic} & 0.88 & 0.82 & 0.77 & 0.36 & 0.43 & 0.51 \\
\rowcolor{babyblue} & \textbf{DC-MMD (ours)} & \textbf{0.93} & \textbf{0.89} & \textbf{0.84} & \textbf{0.24} & \textbf{0.31} & \textbf{0.39} \\ \bottomrule
\end{tabular}
\vspace{-2em}
\end{table*}

An occurrence of an OOD event is declared when
\begin{align}
\label{eq: stopping time}
\tau(b) = \inf \left \{ mt: ~ \max_{1\le i\le t} \sum_{j=i}^t (D(\hat{\mu}_j, \lambda^{(0)}) - \zeta) >b \right \},
\end{align}
where $0 < \zeta < D(\lambda^{(1)}, \lambda^{(0)})$, and the block statistic at block index $k=\lfloor t/m \rfloor$, $W_k:=\max_{1\le i\le k} \sum_{j=i}^k (D(\hat{\mu}_j, \lambda^{(0)}) - \zeta)$, can be computed in an iterative fashion:
\[
W_k = \max \{0,\,W_{k-1}+(D(\hat{\mu}_k, \lambda^{(0)}) - \zeta)\}.
\]
Using MMD in an RKHS ensures the boundedness of the associated random quantities, which in turn enables the application of tools from random process theory (such as Doob's martingale inequality) to derive concentration results.

\begin{remark}[Optional Variance-Normalization]
\label{subsec:adaptive_threshold}
The finite-sample variance of $D (\hat{\mu}_t, \lambda^{(0)})$ may vary substantially across datasets and across different on-board trajectory predictors.
For practical implementation and easier selection of the detection threshold $b$, one may optionally normalize the block increment as
$\frac{D (\hat{\mu}_t, \lambda^{(0)}) -\zeta}{\sqrt{\widehat{\sigma}_t+\varepsilon}}$,
where $\widehat{\sigma}_t$ is estimated from a rolling window of recent blocks, and $\varepsilon>0$ is a small constant introduced to ensure numerical stability.
\end{remark}

Let $\Delta^{(0)} \in [0,1)$ denote the spectral gap of the
transition matrix associated with the second-order Markov chain
under $\Theta^{(0)}$, and let
$R^{(0)} := \sup_{x} \mathbb{E}[k(x, X')^2]$, where $X' \sim \lambda^{(0)}$,
be the second-moment envelope of the kernel.
Both $\Delta^{(0)}$ and $R^{(0)}$ enter the concentration bound for the
blockwise MMD estimator and appear in the detection-delay
guarantee (Theorem~\ref{thm: detection delay}).

\begin{theorem}[Detection delay]
\label{thm: detection delay}
For any given $b>0$, $m\in \mathbb{N}$, and properly chosen $\zeta$ such that $\zeta < D(\lambda^{(1)}, \lambda^{(0)})$,
the WADD of Algorithm \ref{alg:hmm_kernel_cusum} satisfies
\begin{align*}
\text{WADD}(\tau(b)) \le \frac{mb}{(\sqrt{d}-\sqrt{a})^2} + \frac{m \sqrt{d}}{\sqrt{d}-\sqrt{a}},
\end{align*}
where $a=\sqrt{\frac{2-2\Delta^{(0)} + 4 R^{(0)}}{(m-1)(1-\Delta^{(0)})}}$, and $d=D(\lambda^{(1)}, \lambda^{(0)}) -\zeta$.
\end{theorem}
\begin{proof}
Theorem~\ref{thm: detection delay} follows directly from the analysis of \cite[Theorem 1]{zhang2024data} by setting $a_P$, as defined therein, to zero. Since our reference $\lambda^{(0)}$ is estimated from in-distribution data rather than a companion change-free sequence, there is no independent second stream contributing variance; the term $a_P$ in~\cite{zhang2024data} captures the cross-stream covariance of that companion sequence, which vanishes identically under our construction.

\begin{figure*}[t!]
\vspace{0.6em}
    \centering
    \includegraphics[width=\linewidth]{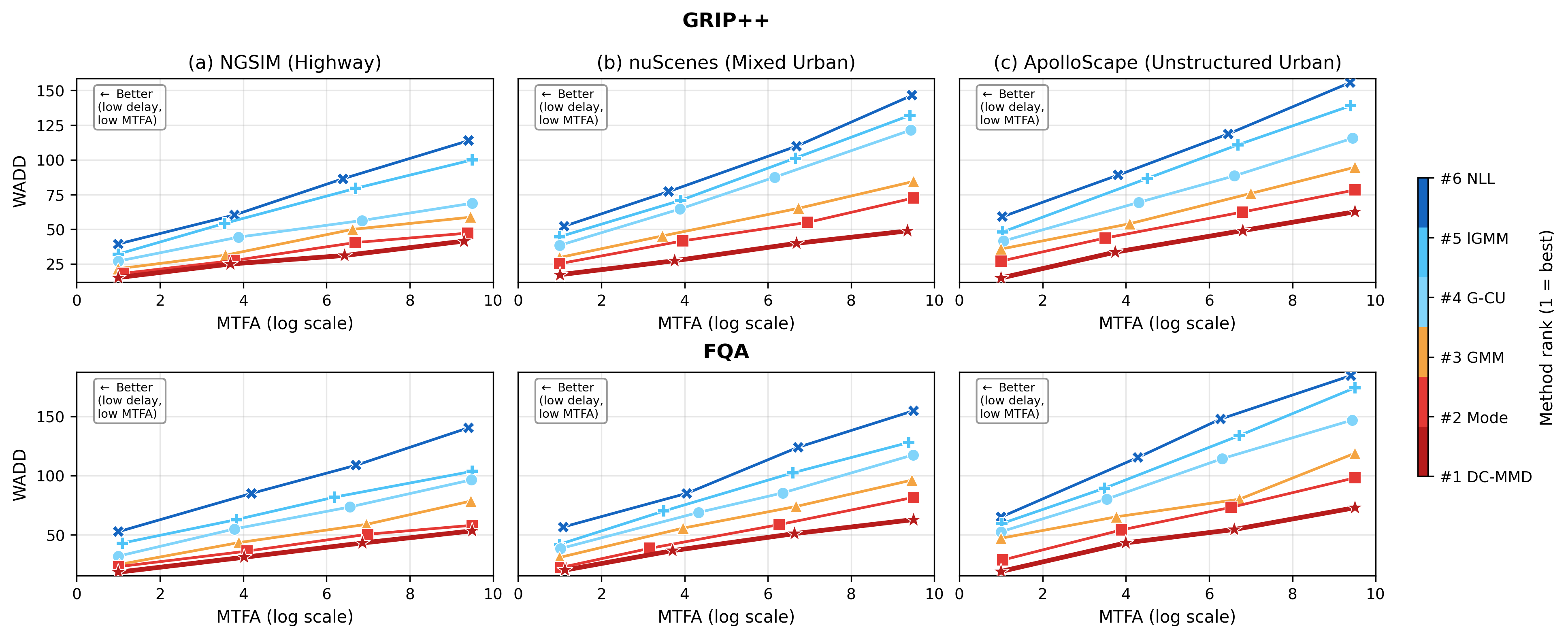}
    \caption{WADD--MTFA trade-off curves across driving domains. Each curve traces one detection method under two models: GRIP++ (top) and FQA (bottom) prediction. WADD down = faster detection; MTFA up = fewer false alarms.}
    \label{fig:rq1-frontiers}
\vspace{-5pt}
\end{figure*}

\end{proof}
\begin{remark}[Tight-bound specialisation]
When $b$ is restricted to a collection of values such that a certain intermediate quantity in the analysis of \cite[Theorem 1]{zhang2024data} is an integer, the second term in Theorem~\ref{thm: detection delay} vanishes and the bound simplifies to $\text{WADD}(\tau(b)) \le \frac{mb}{(\sqrt{d}-\sqrt{a})^2}$.
\end{remark}

\begin{theorem}[False alarm]\cite{zhang2024data}
\label{thm: false alarm}
For any given $b$, $m$, and properly chosen $\zeta$ such that $\zeta < D(\lambda^{(1)}, \lambda^{(0)})$,
the MTFA of Algorithm \ref{alg:hmm_kernel_cusum} grows exponentially in $b$, formally
\[
\mathrm{MTFA}(\tau (b)) \ge m \exp(qb),
\]
where $q>0$ is a constant that depends on $a$ and $\zeta$ but is independent of $b$.
\end{theorem}

\begin{remark}[Order-optimality]
As noted in \cite{zhang2024data}, exponential growth of $\text{MTFA}(\tau(b))$ is desirable.
This is because, to satisfy a given false-alarm constraint $\text{MTFA}(\tau(b))\ge \gamma$ for some $\gamma>0$, one may choose $b = \frac{\log \gamma - \log m}{q}$, assuming that $\gamma$ is a constant that does not depend on other problem parameters.
By Theorem \ref{thm: detection delay}, we know that Algorithm \ref{alg:hmm_kernel_cusum} achieves a detection delay of $O(\frac{\log \gamma - \log m}{q})$, which matches the order-wise lower bound $O(\log \gamma)$ for a broad region of $\gamma$. Hence, Algorithm \ref{alg:hmm_kernel_cusum} achieves order optimality.

\end{remark}

\section{Experimental Setup}
\label{sec:experiments}
\noindent\textbf{Datasets.}
We evaluate DC-MMD
across three complementary real-world driving domains:
NGSIM~\cite{fhwa2020} (structured highway, US-101 and I-80),
nuScenes~\cite{caesar2020nuscenes} (mixed urban traffic, Boston and Singapore),
and ApolloScape~\cite{huang2018apolloscape} (unstructured dense urban, Beijing).

\smallskip
\noindent\textbf{Trajectory prediction models.}
We evaluate two representative trajectory predictors:
{GRIP++}~\cite{li2020gripplus} (graph-based relational modeling)
and {Fuzzy Query Attention (FQA)}~\cite{kamra2020multi}
(attention-based interaction modeling).
The two architectures belong to distinct paradigm families. Consistent DC-MMD performance across both architecturally distinct paradigms confirms model-agnostic runtime monitoring.

\smallskip
\noindent\textbf{Baselines.}
All methods operate on the identical scalar prediction error sequence
$\{e_t\}$ without model introspection or repeated forward
passes. We compare against classic UQ-based detection
(NLL~\cite{lee2018simple}, lGMM~\cite{wiederer2023joint}), and sequential change detectors:
Classical G-CUSUM~\cite{page1954continuous},
GMM-CUSUM~\cite{guo2024building}, and
Mode-Aware CUSUM~\cite{guo2025dynamic}.
The adopted baselines all operate on the scalar sequence $\{e_t\}$ with formal WADD/MTFA guarantees, enabling a fair comparison (see Sec.~\ref{subsec:ood_qcd}).

\begin{table}[t!]
\centering
\caption{Impact of subtle OOD generation on prediction accuracy.
Dev.\ denotes the average spatial shifts
(${\leq}\,1$\,m). $\Delta$ (\%) denotes the error (ADE) increase from ID to OOD setting.}
\label{tab:ood-impact}
\footnotesize
\setlength{\tabcolsep}{5pt}
\begin{tabular}{@{}llccc@{}}
\toprule
\textbf{Dataset} & \textbf{Model}
  & \textbf{Dev.\ (m)} & \textbf{ADE (ID\,/\,OOD)}
  & $\boldsymbol{\Delta}$ \textbf{(\%)} \\
\midrule
\multirow{2}{*}{NGSIM}
  & GRIP++ & 0.46 & 0.73\,/\,1.82 & +149 \\
  & FQA    & 0.44 & 0.81\,/\,2.05 & +153 \\
\midrule
\multirow{2}{*}{nuScenes}
  & GRIP++ & 0.52 & 1.12\,/\,2.71 & +142 \\
  & FQA    & 0.49 & 1.24\,/\,3.01 & +143 \\
\midrule
\multirow{2}{*}{ApolloScape}
  & GRIP++ & 0.55 & 1.48\,/\,3.74 & +153 \\
  & FQA    & 0.53 & 1.61\,/\,4.02 & +150 \\
\midrule
\multicolumn{2}{@{}l}{Average}
  & 0.50 & 1.17\,/\,2.89 & +148 \\
\bottomrule
\multicolumn{5}{l}{\footnotesize * Similar trends
are observed for FDE and RMSE.}
\end{tabular}
\vspace{-1em}
\end{table}

\begin{table}[h]
\centering
\caption{{Detection delay (WADD; in time steps) at MTFA $\approx 2$ under heavy-tailed emission misspecification.}
$\Delta$ (\%) denotes degradation relative to the Gaussian baseline.
}
\label{tab:rq2-misspec}
\footnotesize
\setlength{\tabcolsep}{3pt}
\begin{tabular}{lccc cc}
\toprule
 & \multicolumn{3}{c}{\textbf{WADD}$\downarrow$} & \multicolumn{2}{c}{\textbf{$\Delta$ (\%)}} \\
\cmidrule(lr){2-4}\cmidrule(lr){5-6}
\textbf{Method}
& {\bf Gaussian} & {\bf Laplace} & {\bf Student-$t$}
& LP & ST \\
\midrule
\rowcolor{babyblue}
\textbf{DC-MMD (Ours)} & {\bf 17.52} & {\bf 20.50} & {\bf 20.67} & \textbf{+17} & \textbf{+18} \\
Mode-CUSUM & 26.87 & 31.97 & 34.39 & +19 & +28 \\
GMM-CUSUM  & 36.18 & 43.42 & 47.03 & +20 & +30 \\
G-CUSUM    & 41.35 & 57.06 & 62.85 & +38 & +52 \\
lGMM       & 47.58 & 59.95 & 66.14 & +26 & +39 \\
NLL        & 59.39 & 73.05 & 80.77 & +23 & +36 \\
\bottomrule
\multicolumn{6}{l}{\footnotesize * LP = Laplace, ST = Student-$t$.}
\end{tabular}
\vspace{-5pt}
% \vspace{-2em}
\end{table}

\smallskip
\noindent\textbf{Evaluation Metrics.}
Following the Lorden minimax formulation
(Sec.~\ref{sec:problem}),
we report WADD (Eq. \eqref{eq:wadd})
and MTFA (Eq. \eqref{eq:mtfa}) as primary metrics,
with AUROC \cite{hanley1982meaning} for threshold-free separability between ID and OOD scores, and FPR@95 \cite{hendrycks2017baseline} for the safety-critical high-recall operating point.

\smallskip
\noindent\textbf{OOD Scene Generation.}
Following the controlled-shift method of~\cite{zhang2022adversarial}, we build {\it subtle OOD} scenes by applying bounded perturbations to observed trajectories. These are not artificial noise: they model the observable effect of real hazards---sensor failures, emergency-vehicle cut-ins, snow-covered roads, and unseen pedestrian behavior---that move agents while keeping the scene plausible. Such shifts are the most dangerous and hardest to detect, and even small displacements degrade prediction and planning~\cite{zhang2022adversarial}. We keep velocity, acceleration, and jerk within $\mu\pm3\sigma$ of training, and cap spatial deviation at 1\,m per point to preserve lane consistency. Since a lane is ${\approx}3.5$\,m wide, this bound stays physically valid yet still causes the large accuracy drop in Table~\ref{tab:ood-impact}; larger shifts would break lane semantics and make detection trivial. Perturbations averaging only ${\approx}0.5$\,m already degrade accuracy sharply across datasets and models. Once a shift begins, the per-block MMD statistic $D(\hat{\mu}_t,\lambda^{(0)})$ exceeds the offset $\zeta$, the CUSUM score $W_k$ builds up, and an alarm fires when $W_k$ crosses $b$---within the WADD bound of Theorem~\ref{thm: detection delay}.

\smallskip
\noindent\textbf{Implementation Details.}
All methods operate on the same scalar error sequence to ensure fair
comparison. The hyperparameters of DC-MMD
(Algorithm~\ref{alg:hmm_kernel_cusum}) are set as follows.
We set the block length to $m{=}50$, the best value from $\{25,50,75,100\}$.
Smaller $m$ is noisier; larger $m$ detects later; $m{=}50$ balances the two. The kernel $k$ is a Gaussian RBF with bandwidth $\sigma{=}0.8$ chosen by the median heuristic of~\cite{gretton2012kernel}; specifically, it matches the median pairwise distance of second-order ID samples pooled across training splits.
The offset $\zeta{=}0.05$, corresponding to the sample mean of MMD values on ID blocks, so that the CUSUM statistic $W_t$ has a negative drift under nominal operation.
For the detection delay versus false
alarm trade-off curves, the threshold $b$ is swept over a range so
that all methods are compared at matched MTFA.

\section{Results}
\label{sec:results}
For conciseness, results are reported using GRIP++ with ADE as the representative predictor--metric pair. Similar trends hold across other configurations.

\begin{figure*}[t!]
\vspace{0.6em}
    \centering
    \includegraphics[trim={0.5cm 0 1.8cm 0}, clip, width=\linewidth]{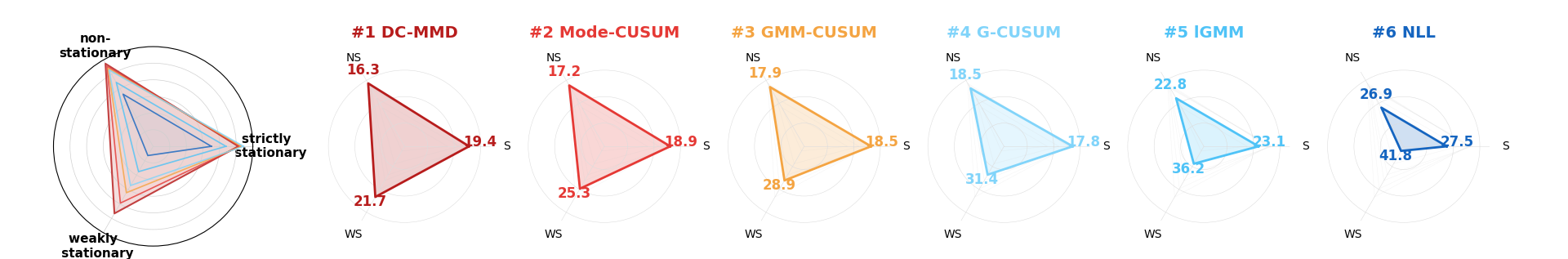}
    \caption{Latent-dynamic performance across benchmarks. Each radar axis represents a progressively weaker stationarity assumption, arranged clockwise: \textbf{S} (strictly stationary) $\to$ \textbf{WS} (weakly stationary) $\to$ \textbf{NS} (non-stationary). The Global Profile (left) overlays all six methods; individual panels isolate each method's footprint with per-axis average WADD annotated (at MTFA $\approx 2$ ; 500 runs).  Larger area $\Leftrightarrow$ WADD$\downarrow$ (faster change-point response).}
    \label{fig:rq3_radar_expanded}
\end{figure*}

\begin{figure*}[t]
    \centering
    \includegraphics[trim={1cm 0cm 0cm 0cm}, clip, width=1\textwidth]{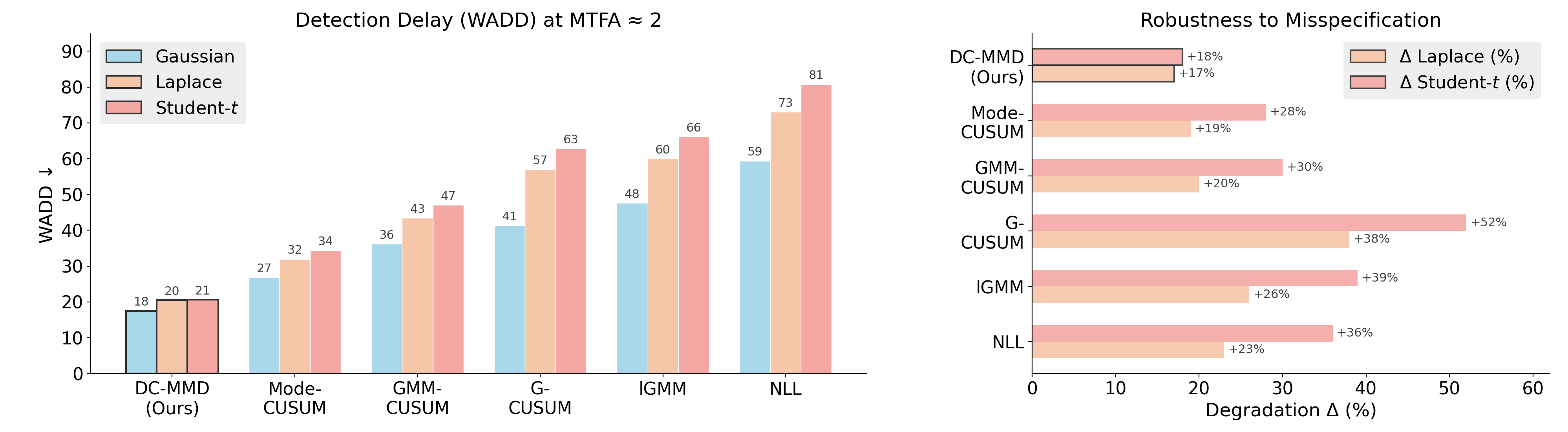}
    \caption{{Robustness to heavy-tailed distribution.}
    \textit{Left:} Detection delay (WADD) at $\text{MTFA} \approx 2$. {DC-MMD (Ours)} maintains the lowest delay across all distributions:
    Gaussian (light-tailed, {\color[HTML]{A9D0E5} blue}), Laplace (moderate-tailed, {\color[HTML]{EBC4A3} tan}), and Student-$t$ (heavy-tailed, {\color[HTML]{E6A39E} red}).}
    \label{fig:heavy_tail_robustness}
    \vspace{-0.5em}
\end{figure*}

\subsection{{How does DC-MMD compare with baselines
  across predictors and driving domains?}}
\label{subsec:rq1}
We evaluate this question using detection accuracy (Table~\ref{tab:rq1-performance}), delay--false-alarm trade-offs (Fig.~\ref{fig:rq1-frontiers}), and robustness under different latent dynamics (Fig.~\ref{fig:rq3_radar_expanded}).
DC-MMD has the strongest accuracy in every domain, reaching AUROC $0.84$ on ApolloScape while lowering FPR@95 below sequential baselines like Mode-CUSUM; point-wise methods (NLL, lGMM) stay near $0.6$, weak in high-entropy scenes. The frontier in Fig.~\ref{fig:rq1-frontiers} shows DC-MMD cutting WADD by $15$--$35\%$ across all three datasets at comparable MTFA. Fig.~\ref{fig:rq3_radar_expanded} adds that it keeps the lowest delay across i.i.d., static multi-modal, and Markov-switching regimes, where classical detectors fall off sharply. Notably, The WS axis is the hardest for most baselines because weakly stationary processes drift slowly while their marginals look in-distribution; DC-MMD avoids this by monitoring the joint distribution of $(e_{t-1}, e_t)$, which reflects transition-level shifts first.

\subsection{How robust is DC-MMD under heavy-tailed distribution?}
\label{subsec:rq2}

Table~\ref{tab:rq2-misspec} tests robustness to heavy-tailed emission misspecification at MTFA $\approx 2$, where DC-MMD already has the lowest delay. Likelihood-based detectors degrade most under Laplace and Student-$t$ emissions: G-CUSUM rises by $38\%$ and $52\%$, and Mode-CUSUM and GMM-CUSUM by $19$--$30\%$. DC-MMD stays stable, rising only $+17\%$ and $+18\%$---the smallest degradation of all methods (Fig.~\ref{fig:heavy_tail_robustness}).

\subsection{How robust is DC-MMD to unknown post-change distributions?} \label{subsec:rq3}

\begin{table}[t!]
\centering
\caption{{Detection delay (WADD; in time steps) when the post-change distribution is unknown.}
All methods operate at MTFA $\approx 1$.
Results are averaged over 500 independent runs.
Robust CUSUM uses a shift-based surrogate with $\kappa{=}2$ selected following~\cite{guo2024building}. }
\label{tab:rq3-unknown}
\footnotesize
\setlength{\tabcolsep}{5pt}
\begin{tabular}{lcc}
\toprule
\textbf{Method} & \textbf{Robust} & \textbf{Delay} \\
\midrule
\rowcolor{babyblue}
\textbf{DC-MMD (Ours)} & \checkmark & \textbf{16.7} \\
Robust CUSUM ($\kappa$-shift) & \checkmark & 16.4 \\
G-CUSUM (misspecified) & & 25.8 \\
NLL & & 32.6 \\
\bottomrule
\end{tabular}
% \vspace{-0.5em}
\end{table}

\begin{table}[t!]
\centering
\caption{Computational Complexity and Real-Time Suitability.}
\label{tab:complexity}
\resizebox{\columnwidth}{!}{%
\begin{tabular}{@{}lccc@{}}
\toprule
\textbf{Method} & \textbf{Complexity (Per Step)} & \textbf{Parameters} & \textbf{Real-Time} \\ \midrule
NLL~\cite{lee2018simple} & $\text{O}(1)$ & 0 & $\star \star \star$ \\
G-CUSUM~\cite{page1954continuous} & $\text{O}(1)$ & 0 & $\star \star \star$ \\
lGMM~\cite{wiederer2023joint} & $\text{O}(1)$ & Low & $\star \star \star$ \\
Deep Ensembles~\cite{lakshminarayanan2017simple} & $\text{O}(N)$ & $N \times$ Model & $\star$ \\ \midrule
\rowcolor{babyblue}
\textbf{DC-MMD (Ours)} & $\mathbf{\text{O}(1)^\ddagger}$ & \textbf{Minimal} & $\mathbf{\star \star \star}$ \\ \bottomrule
\end{tabular}%
}
\begin{flushleft}
\footnotesize $\star \star \star$ Excellent ($<$ 0.01ms), $\star$ Low ($>$ 10ms). \\
$^\ddagger$ Achieved via recursive sliding-window MMD updates.
\end{flushleft}
\vspace{-10pt}
% \vspace{-3em}
\end{table}

We next test an unknown post-change distribution, where DC-MMD needs no parametric assumption. As a baseline for this case we add Robust CUSUM, which assumes a minimal shift $\kappa{=}2$~\cite{guo2024building}, plus G-CUSUM and NLL, which have no robust design. In Table~\ref{tab:rq3-unknown}, DC-MMD and Robust CUSUM are comparable (16.7 vs.\ 16.4). G-CUSUM is slower (25.8) since its likelihood ratio is misspecified once the distribution leaves the assumed model, and NLL is slowest (32.6) for lacking CUSUM's optimal stopping. But Robust CUSUM matches us only when $\kappa$ is correct; with an unknown shift or non-Gaussian errors, it degrades too. DC-MMD avoids this: its kernel statistic is valid for any continuous distribution. We thus recommend it when the post-change regime is non-stationary or the shift cannot be pre-calibrated.

\subsection{Is DC-MMD deployable under real-time constraints?}
\label{subsec:rq4}

To evaluate the feasibility of DC-MMD for high-frequency robotics, we compare its per-step computational overhead against baselines in Table~\ref{tab:complexity}. Our design is decoupled from the primary model's parameter space, ensuring constant-time inference regardless of the predictor's depth.

\section{Conclusion and Discussion}
\label{sec:conclusion}
We presented {DC-MMD}, a latent-aware sequential change detector for runtime monitoring of trajectory prediction. It models pre-change errors as an HMM to obtain the reference $\lambda^{(0)}$, then applies a CUSUM rule to block-wise MMD statistics. Across NGSIM, nuScenes, and ApolloScape, DC-MMD cuts detection delay by 15--35\% over the strongest sequential baseline, stays stable under misspecification, and matches robust baselines in distribution-free settings---all at $O(1)$ per-step cost. It serves as a runtime monitor that triggers a fallback when OOD conditions appear; in-distribution calibration stays with the base predictor's UQ component, so the two are complementary. Future work will add perception-level signals and link detection to planner-level fallback for end-to-end runtime assurance.

\begingroup
\renewcommand{\baselinestretch}{0.95}
\scriptsize
\bibliographystyle{IEEEtran}
\bibliography{references}
\endgroup

\end{document}